\documentclass[runningheads]{llncs}
\usepackage[T1]{fontenc}
\usepackage{graphicx}
\usepackage{xcolor}
\usepackage{amssymb}
\usepackage{amsfonts}
\usepackage{amsmath}
\usepackage{multirow}
\usepackage{cite}
\usepackage{caption}
\usepackage{subcaption}

\usepackage{rotating}
\usepackage[misc]{ifsym}
\usepackage[marginal]{footmisc}

\newcommand{\eat}[1]{}
\newcommand{\etal}{{et al.~}}       
\newcommand{\eg}{{e.g.~}}           
\newcommand{\ie}{{i.e.~}}           
\newcommand{\wrt}{{w.r.t.~}}         
\newcommand{\aka}{{a.k.a.~}}        

\makeatletter
\newcommand*{\indep}{%
	\mathbin{%
		\mathpalette{\@indep}{}%
	}%
}
\newcommand*{\nindep}{%
	\mathbin{
		\mathpalette{\@indep}{\not}
	}%
}
\newcommand*{\@indep}[2]{%
	\sbox0{$#1\perp\m@th$}
	\sbox2{$#1=$}
	\sbox4{$#1\vcenter{}$}
	\rlap{\copy0}
	\dimen@=\dimexpr\ht2-\ht4-.2pt\relax
	\kern\dimen@
	{#2}%
	\kern\dimen@
	\copy0 
}

\begin{document}
\title{Learning Conditional Instrumental Variable Representation for Causal Effect Estimation}
\titlerunning{DVAE.CIV}
\author{Debo Cheng\inst{1}\textsuperscript{*}(\Letter)\and
	Ziqi Xu\inst{1}\textsuperscript{*} \and
	Jiuyong Li\inst{1} \and
	Lin Liu\inst{1}\and \\
	Thuc Duy Le\inst{1} \and
	Jixue Liu\inst{1}}
\authorrunning{D. Cheng et al.}
%
\institute{\textsuperscript{1}University of South Australia, Adelaide, Australia\\ 			  
	\email{\{Debo.Cheng,Ziqi.Xu\}@mymail.unisa.edu.au,}\\
	\email{\{Jiuyong.Li,Lin.Liu,Thuc.Le,Jixue.Liu\}@unisa.edu.au}}
\maketitle              

\footnote{* These authors contributed equally.}

\begin{abstract}
One of the fundamental challenges in causal inference is to estimate the causal effect of a treatment on its outcome of interest from observational data. However, causal effect estimation often suffers from the impacts of confounding bias caused by unmeasured confounders that affect both the treatment and the outcome. The instrumental variable (IV) approach is a powerful way to eliminate the confounding bias from latent confounders. However, the existing IV-based estimators require a nominated IV, and for a conditional IV (CIV) the corresponding conditioning set too, for causal effect estimation. This limits the application of IV-based estimators. In this paper, by leveraging the advantage of disentangled representation learning, we propose a novel method, named DVAE.CIV, for learning and disentangling the representations of CIV and the representations of its conditioning set for causal effect estimations from data with latent confounders. Extensive experimental results on both synthetic and real-world datasets demonstrate the superiority of the proposed DVAE.CIV method against the existing causal effect estimators.
\end{abstract}

\section{Introduction}
\label{Sec:Intro}
It is a fundamental task to query or estimate the causal effect of a treatment  (\aka exposure, intervention or action) on an outcome of interest in causal inference. Causal effect estimation has wide applications across many fields, including but not limited to, economics~\cite{imbens2015causal}, epidemiology~\cite{hernan2006instruments,martinussen2019instrumental}, and computer science~\cite{pearl2009causality}. The gold standard method for causal effect estimation is randomised controlled trials (RCT), but they are often impractical or unethical due to cost restrictions or ethical constraints~\cite{pearl2009causality,imbens2015causal}. Instead of conducting an RCT, estimating causal effects from observational data offers an alternative to evaluate the effect of a treatment on the outcome of interest.

\begin{figure}[t]
 \centering
 \includegraphics[scale=0.28]{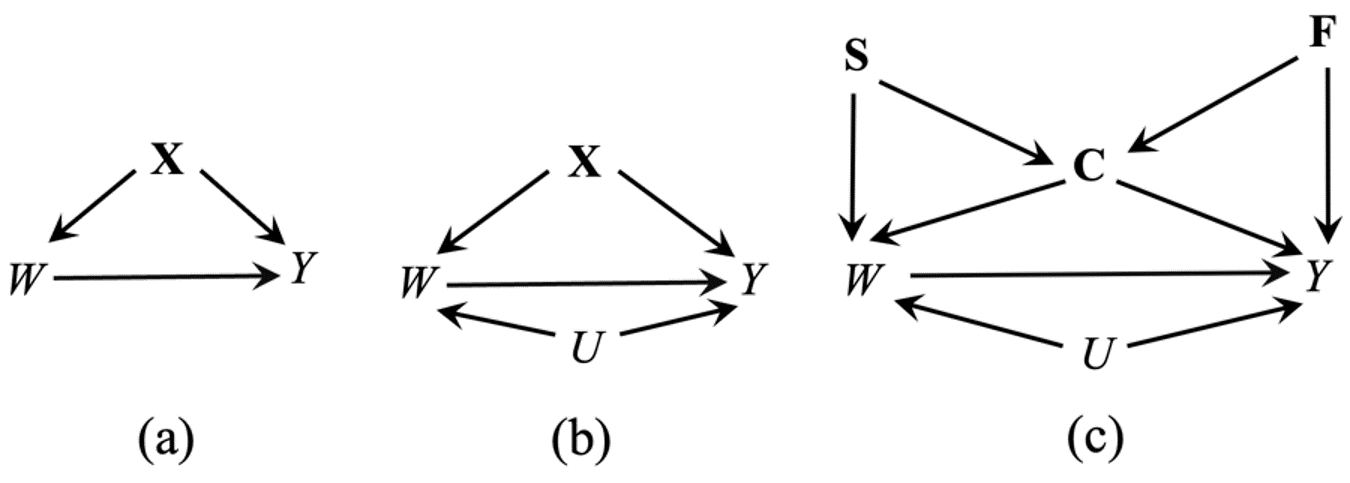}
 \caption{Three causal DAGs are utilised to illustrate the problems of causal effect estimation from observational data. In all three DAGs, $\mathbf{X}$, $U$, $W$ and $Y$ are the set of pretreatment variables, latent confounder, treatment and outcome variables, respectively. (a) indicates the unconfoundedness assumption holding, and (b) shows the causal effect of $W$ on $Y$ is non-identification since there is a latent confounder $U$. (c) illustrates the problem studied in this work, in which the set $\mathbf{X}$ is represented by three sets $\{\mathbf{S, C, F}\}$.} 
\label{fig:dags}
\end{figure}

Confounding bias is a major obstacle in estimating causal effects from observational data. It arises from confounders that affect both the treatment variable $W$ and the outcome variable $Y$. When all confounders are measured (i.e., the unconfoundedness assumption~\cite{rosenbaum1983central,imbens2015causal} is satisfied), adjusting for the set of all measured confounders is sufficient to obtain an unbiased estimation of the causal effect from observational data~\cite{imbens2015causal,abadie2006large}. For example, in the causal graph of Fig.~\ref{fig:dags}(a),  the unconfoundedness is satisfied when given $\mathbf{X}$. Nevertheless, the unconfoundedness assumption is untestable, and there exists a latent (\aka unobserved, unmeasured) confounder affecting both $W$ and $Y$ in many real-world applications, \eg the latent confounder $U$ affects both $W$ and $Y$ in the causal graph in Fig.~\ref{fig:dags}(b). In such a situation, the causal effect of $W$ on $Y$ is non-identification~\cite{pearl2009causality}. Most existing data-driven methods rely on the unconfoundedness assumption and thus it becomes challenging and questionable for them to obtain unbiased causal effects from data with latent confounders.  
 
The instrumental variable (IV) approach is a practical and powerful technique for addressing the challenging problem of causal effect estimation in the presence of latent confounders. The IV approach requires a valid IV for eliminating the confounding bias caused by latent confounders~\cite{angrist1995two,imbens2014instrumental}. Valid IVs are exogenous variables that are associated with $W$ but not directly associated with $Y$~\cite{hernan2006instruments,martens2006instrumental}. A valid IV $S$ needs to satisfy three conditions: (1) $S$ is correlated to $W$; (2) $S$ and $Y$ do not share confounders (\ie \emph{unconfounded instrument}); and (3) the effect of $S$ on $Y$ is entirely through $W$ (\ie \emph{exogenous})~\cite{martens2006instrumental,hernan2006instruments}. However, the last two conditions are too strict and not testable in real-world applications. Therefore, in many existing IV-based methods, an IV is nominated based on prior or domain knowledge. However, in many real-world applications, the nominated IVs based on domain knowledge could violate one of the three conditions, resulting in a biased estimate and potentially leading to incorrect conclusions~\cite{brito2002generalized,hernan2006instruments}. 

It is a challenging problem to discover a valid IV directly from data. Investigators usually collect as many covariates as possible, but few of them are valid IVs that satisfy the three conditions. Instead of discovering a valid IV, Kang~\etal\cite{kang2016instrumental} proposed a data-driven method, referred to as sisVIVE, based on the assumption of some invalid and some valid IVs (\ie more than half of candidate IVs are valid IVs ) to provide a bound of causal effect estimations. Hartford~\etal\cite{hartford2017deep} proposed DeepIV, a deep learning based IV approach for counterfactual predictions, but it requires a nominated IV and the corresponding conditioning set. Kuang~\etal\cite{kuang2020ivy} developed a method to model a summary IV as a latent variable based on the statistical dependencies of the set of candidate IVs. Yuan~\etal\cite{yuan2022auto} proposed a data-driven method to automatically generate a synthetic IV for counterfactual predictions, but the method does not consider the confounding bias between the IV and the outcome, and the condition of \emph{unconfounded instrument} may be violated in many cases. Therefore, it is desirable to develop an algorithm for learning a valid IV that considers the unconfounded instrument for causal effect estimations, especially conditional average causal effect estimations, from data with latent confounders.

To provide a practical solution for conditional average causal effect estimations, in this work, we focus on conditional IV (CIV), which can be considered as an IV with relaxed conditions and a CIV requires a conditioning set to instrumentalise it to function as an IV (details see Definition~\ref{def:conditionalIV}). We propose to leverage disentangled representation learning technique to learn from data the representations of a CIV and its conditioning set.  

Specifically, as shown in Fig.~\ref{fig:dags}(c), we assume that the observed covariates are learned through three representations, $\mathbf{S}$, $\mathbf{C}$ and $\mathbf{F}$. Here, $\mathbf{S}$ affects both treatment $W$ and $\mathbf{C}$, $\mathbf{C}$ represents the confounding factor affecting both $W$ and the outcome $Y$, and $\mathbf{F}$ represents the risk factor affecting both $\mathbf{C}$ and $Y$.  We then establish a theorem that $\mathbf{S}$ is a valid CIV that is instrumentalised by $\{\mathbf{C, F}\}$, meaning that $\{\mathbf{C, F}\}$ is the conditioning set of $\mathbf{S}$.  Supported by this theorem, we design and develop a novel disentangled representation learning algorithm called DVAE.CIV model, which is based on the Variational AutoEncoder (VAE) model~\cite{kingma2014auto}. This model allows us to obtain the representations of the CIV $\mathbf{S}$ and its conditioning set $\{\mathbf{C, F}\}$, enabling us to use $\mathbf{S}$ as a valid IV conditioning on $\{\mathbf{C, F}\}$  for estimating the conditional average causal effects of $W$ on $Y$ from data when there are latent confounders. The main contributions of the paper are summarised as follows.
\begin{itemize}
	\item We address a challenging problem in conditional average causal effect estimations from data with latent confounders by utilising the CIV approach and VAE models. 
	\item We propose a novel disentanglement learning model based on the conditional VAE model to learn and disentangle the representations of covariates into the representations of a CIV $\mathbf{S}$ and its conditioning set $\{\mathbf{C, F}\}$ for conditional average causal effect estimations from data with latent confounders. 
	\item We conduct extensive experiments on synthetic and real-world datasets to show the performance of the DVAE.CIV model, \wrt causal effect estimations from data with latent confounders. 
\end{itemize}

\section{Preliminaries}
\label{sec:pre}
In this paper,  uppercase and lowercase letters are utilised to represent variables and their values, respectively. Bold-faced uppercase and lowercase letters indicate a set of variables and a value assignment of the set, respectively. 

A DAG (direct acyclic graph) is a graph that contains directed edges (\ie $\rightarrow$) without cycles. In a DAG $\mathcal{G}$, the directed edge $X_i \rightarrow X_j$ represents that $X_i$ is a cause of $X_j$, and $X_j$ is an effect of $X_i$. A DAG is a causal DAG when a direct edge $X_i \rightarrow X_j$ represents that $X_i$ is a cause of $X_j$. In this work, we assume a causal DAG $\mathcal{G}=(\mathbf{V, E})$ to represent the underlying system, where $\mathbf{V}=\mathbf{X}\cup\mathbf{U}\cup\{W, Y\}$, and $\mathbf{E}\subseteq \mathbf{V} \times \mathbf{V}$ denotes directed edges. In $\mathbf{V}$, we assume that $\mathbf{X}$ is the set of pretreatment variables, $\mathbf{U}$ is the set of latent confounders, $W$ is a binary treatment variable ($w=1$ and $w=0$ denote the treated sample and control sample, respectively), and $Y(w)$ is an outcome of interest. Following the potential outcome model~\cite{rosenbaum1983central,imbens2015causal}, we have the potential outcomes $Y(w=1)$ and $Y(w=0)$ relative to the treatment $W$. Note that we can only measure one of the two potential outcomes for a given individual $x_i$. Conceptually,  the individual causal effect (ICE) at $x_i$ is defined as $ICE_i = Y_i(w=1) -Y_i(w=0)$.  The average causal effect of $W$ on $Y$ is defined as ACE$(W, Y)= \mathbb{E}[Y_i(w=1) -Y_i(w=0)]$, where $\mathbb{E}$ is the expectation function.

The conditional average causal effect (CACE) of $W$ on $Y$ is referred to as CACE$(W, Y)$, and defined as the form $P(Y| do(w), \mathbf{X})$, where $do(\cdot)$ is do-operation and indicates an intervention on the treatment (\ie set the value of $W$ as per~\cite{pearl2009causality}). Conceptually, $P(Y| do(w), \mathbf{X})$ can be obtained as:
\begin{equation}
	\label{eq003}
	\begin{aligned}
	{\rm CACE}(W, Y) =  \mathbb{E}[Y_i (w=1) - Y_i(w=0)\mid \mathbf{x}_i =x]  
	\end{aligned}
\end{equation}

In this work, we would like to estimate CACE$(W, Y)$ from data that there exists at least a latent confounder $U$ affecting both $W$ and $Y$.  When there is an IV $S$ and the set of conditioning covariates $\mathbf{Z}$ available in data, CACE$(W, Y)$ can be calculated by the following formula as in~\cite{angrist1996identification,imbens2015causal}:
\begin{equation}
	\label{eq004}
	{\rm CACE}(W, Y)= \frac{\mathbb{E}(Y|W = 1, S = 1, \mathbf{Z}) - \mathbb{E}(Y|W = 0, S = 1, \mathbf{Z})}{\mathbb{E}(W|S = 1, \mathbf{Z}) - \mathbb{E}(W|S = 0, \mathbf{Z})}
\end{equation}

The approach of CIV allows a measured covariate to be a valid IV conditioning on a set of measured variables. The formal definition of the CIV in a DAG (Definition 7.4.1 on Page 248~\cite{pearl2009causality}) is introduced as follows.
\begin{definition}[Conditional IV]
	\label{def:conditionalIV}
	Let $\mathcal{G}=(\mathbf{V, E})$ be a DAG with $\mathbf{V}=\mathbf{X}\cup\mathbf{U}\cup\{W, Y\}$, a variable $Q\in\mathbf{X}$ is a conditional IV \wrt $W\rightarrow Y$ if there exists a set of measured variables $\mathbf{Z}\subseteq\mathbf{X}$ such that (i) $Q\nindep_d W\mid\mathbf{Z}$, (ii) $Q\indep_d Y\mid\mathbf{Z}$ in $\mathcal{G}_{\underline{W}}$, and (iii) $\forall Z\in \mathbf{Z}$, $Z$ is not a descendant of $Y$.
\end{definition}

Here, $\indep_d$ and $\nindep_d$ are d-separation and d-connection for reading the conditioning relationships between nodes in a DAG~\cite{pearl2009causality}. The manipulated DAG $\mathcal{G}_{\underline{W}}$ in Definition~\ref{def:conditionalIV} is obtained by deleting the direct edge $W\rightarrow Y$ from the DAG $\mathcal{G}$. Note that Definition~\ref{def:conditionalIV} is defined on a single CIV $Q$ that can be generalised to a set of CIVs $\mathbf{Q}$ easily.

With the pretreatment variables assumption, there is not a descendant of $Y$ in $\mathbf{X}$, \ie the condition (iii) of Definition~\ref{def:conditionalIV} is always held. It means that one needs to check the first two conditions for verifying whether a variable is a CIV or not.  Note that discovering a conditioning set $\mathbf{Z}$ from a given DAG is NP-complete~\cite{van2015efficiently}. Under the pretreatment assumption, the time complexity of discovering a conditioning set is still NP-complete. Instead of discovering a conditioning set from a given causal DAG, in this work, we will utilise disentangled representation learning to learn the representations of CIVs and the representations of the conditioning set directly from data with latent confounders. 

\section{The Proposed DVAE.CIV Model}
\label{sec:theo}
\subsection{The Disentangled Representation Learning Scheme for Causal Effect Estimation}
\label{subsec:thedisentanglingRLS}

In this work, we would like to estimate CACE$(W, Y)$ from observational data with latent confounders. Note that the causal effect of $W$ on $Y$ is non-identifiable when there exists a latent confounder $U\in \mathbf{U}$ affecting both $W$ and $Y$, \ie $W\leftarrow U \rightarrow Y$ in the underlying DAG~\cite{brito2002generalized,pearl2009causality}. It is challenging to recover CACE$(W, Y)$ from data with latent confounders due to the effect of $U$ is not computable. If there is a nominated CIV and its corresponding conditioning set, CACE$(W, Y)$ can be obtained unbiasedly from data by using an IV-based estimator. However, a CIV and its conditioning set are usually unknown in many real-world applications. Furthermore, if an invalid CIV is used, the wrong result or conclusion may be drawn~\cite{martens2006instrumental,cheng2022data}.   

To estimate the conditional average causal effects and average causal effects from data with latent confounders, we propose and design the DVAE.CIV model to learn three representations $\{\mathbf{S, C, F}\}$ as in the scheme of Fig.~\ref{fig:dags}(c). Here $\mathbf{S}$ is the representation of CIVs that only affect $W$ but not $Y$, $\mathbf{F}$ is the representation of the risk factors that affects $Y$ but not $W$, and $\mathbf{C}$ is the confounding representation that affecting both $W$ and $Y$.

Our proposed DVAE.CIV model relies on VAEs: we assume that the measured covariates factorise conditioning on the latent variables, and use an inference model~\cite{kingma2014auto}  which follows a factorisation of the true posterior~\cite{louizos2017causal,hassanpour2019learning}. Based on our disentanglement setting in Fig.~\ref{fig:dags}(c), we have the following theoretical result for causal effect estimation from data with latent confounders. 
 
\begin{theorem} 
	\label{theo:001}
	Let $\mathcal{G}=(\mathbf{X}\cup \mathbf{U}\cup\{W, Y\}, \mathbf{E})$ be a causal DAG, in which $\mathbf{X}$ is a set of pretreatment variables, $\mathbf{U}$ is a set of latent confounders, $W$ and $Y$ are treatment and outcome variables, respectively, and $W\rightarrow Y$ is in $\mathbf{E}$. If we can learn the three representations as per the scheme in Fig.~\ref{fig:dags}(c), then the quantities of  {CACE}$(W, Y)$ can be calculated by using IV-based method.
\end{theorem}
\begin{proof}
	The directed edge $W\rightarrow Y$ in $\mathcal{G}$ is to ensure that  $W$ has a causal effect on $Y$.  In the causal DAG in Fig.~\ref{fig:dags}(c), we first show that the set $\mathbf{C}\cup\mathbf{F}$ instrumentalises $\mathbf{S}$ to be a valid CIV. $\mathbf{S}$ is a common cause of $W$ and $\mathbf{C}$, so $S\nindep_d W$, \ie the first condition of Definition~\ref{def:conditionalIV} holds.  In the causal DAG $\mathcal{G}$ in Fig.~\ref{fig:dags}(c), $\mathbf{C}$ is a collider and is a common cause of $W$ and $Y$. That is, conditioning on $\mathbf{C}$, the path $W\leftarrow \mathbf{S}\rightarrow \mathbf{C}\leftarrow \mathbf{F}\rightarrow Y$ is open, but $\mathbf{F}$ is sufficient to block this path. For the path $\mathbf{S}\rightarrow \mathbf{C}\rightarrow Y$, $\mathbf{C}$ blocks it. Furthermore,  in the manipulated DAG $\mathcal{G}_{\underline{W}}$, $W$ is a collider such that the empty set blocks the three paths between $\mathbf{S}$ and $Y$, \ie $\mathbf{S}\rightarrow W\leftarrow U\rightarrow Y$, $\mathbf{S}\rightarrow W\leftarrow \mathbf{C}\leftarrow\mathbf{F}\rightarrow Y$ and $\mathbf{S}\rightarrow W\leftarrow\mathbf{C}\rightarrow Y$. Hence, the set $\mathbf{C}\cup\mathbf{F}$ blocks all paths between $\mathbf{S}$ and $Y$ in $\mathcal{G}_{\underline{W}}$, \ie the second condition of Definition~\ref{def:conditionalIV} holds. Finally,  $\mathbf{C}\cup\mathbf{F}$ does not contains a descendant of $Y$ due to the pretreatment variables assumption. Thus, the set $\mathbf{C}\cup\mathbf{F}$ instrumentalises $\mathbf{S}$. As in Eq.(\ref{eq004}), the IV-based estimators, such as DeepIV~\cite{hartford2017deep}, can be applied to remove the effect of $\mathbf{U}$ by inputting the CIV representation $\mathbf{S}$ and the representations of its conditioning set $\mathbf{C}\cup\mathbf{F}$. 
	Therefore, the quantities of  $\rm{CACE}(W, Y)$ can be obtained by using the CIV $\mathbf{S}$ and its conditioning set $\mathbf{C}\cup\mathbf{F}$ in an IV-based estimator. 
\end{proof}

Theorem~\ref{theo:001} ensures that a family of data-driven methods can be applied for causal effect estimation from data with latent confounders.

\subsection{Learning the Three Representations} 
\label{subsec:learningthree}

Based on Theorem~\ref{theo:001}, we have known that the set $\{\mathbf{C}, \mathbf{F}\}$ instrumentalises $\mathbf{S}$.  In this section, we present our proposed DVAE.CIV model for obtaining the three latent representations from data by using the VAE technique~\cite{kingma2014auto}, and the architecture of DVAE.CIV is presented in Fig.~\ref{fig:DVAE.CIV_arch}. As shown in Fig.~\ref{fig:DVAE.CIV_arch}, the DVAE.CIV model is to learn and disentangle the latent representation $\mathbf{\Phi}$ of $\mathbf{X}$ into two disjoint sets $\{\mathbf{S, F}\}$ by using disentangled variational autoencoder~\cite{hassanpour2019learning,zhang2021treatment}, and generate the representation $\mathbf{C}$ conditioning on $\mathbf{X}$ by jointing the Conditional Variational AutoEncoder (CVAE) network~\cite{sohn2015learning}. 

\begin{figure}[t]
	\centering
	\begin{subfigure}[b]{0.38\textwidth}
		\centering
		\includegraphics[scale=0.25]{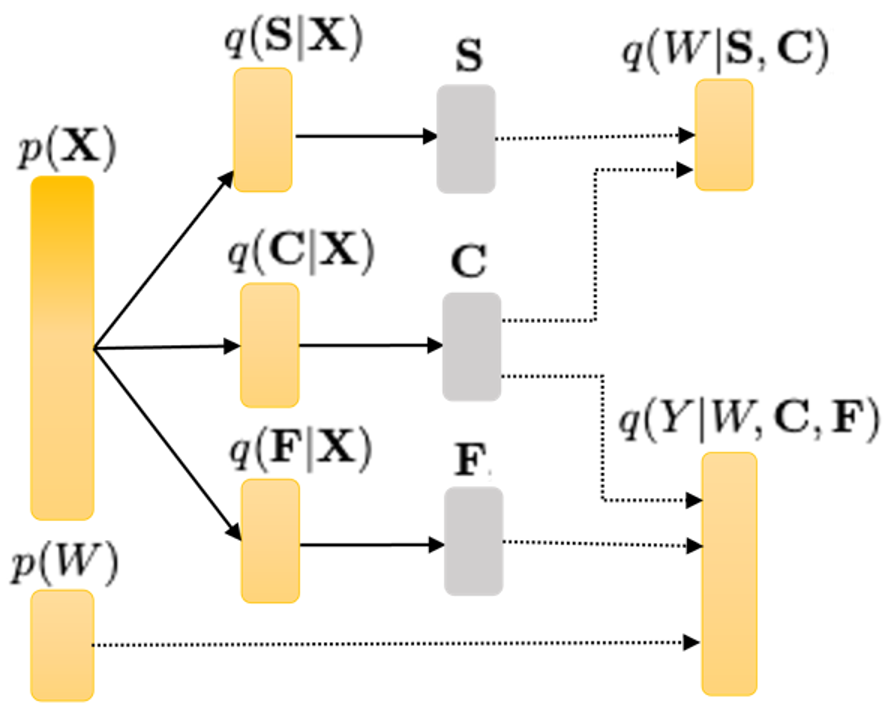}
		\caption{}
		\label{pic:intro2}
	\end{subfigure}
	\begin{subfigure}[b]{0.59\textwidth}
		\centering
		\includegraphics[scale=0.25]{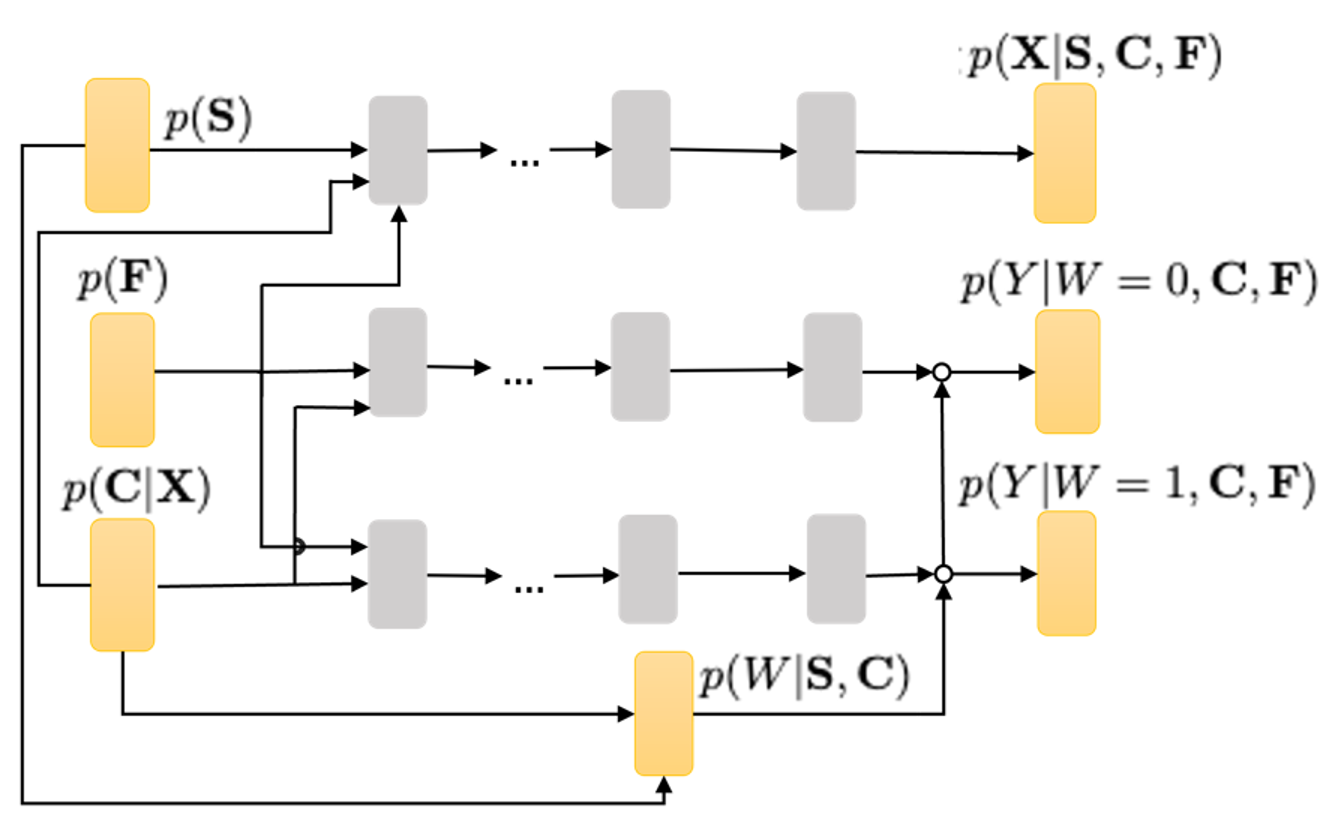}
		\caption{}
		\label{pic:intro3}
	\end{subfigure}
	\caption{The architecture of DVAE.CIV model. A yellow box indicates the drawing of samples from the respective distributions, a grey box indicates the parameterised deterministic neural network transitions, and a circle represents switching paths based on the value of $W$.}
	\label{fig:DVAE.CIV_arch}
\end{figure}

The DVAE.CIV model is designed to learn three representations shown in Fig.\ref{fig:dags}(c) by utilising the inference model and generative model to approximate the posterior distribution $p(\mathbf{X}|\mathbf{S},\mathbf{C},\mathbf{F})$. The inference model comprises three independent encoders $q(\mathbf{S}|\mathbf{X})$, $q(\mathbf{C}|\mathbf{X})$, and $q(\mathbf{F}|\mathbf{X})$, which are treated as variational posteriors over the three latent representations. The generative model utilises the three latent representations with a decoder model $p(\mathbf{X}|\mathbf{S},\mathbf{C}, \mathbf{F})$ to reconstruct the measured distribution $\mathbf{X}$.

Following the standard VAE model~\cite{kingma2014auto}, the prior distributions $p(\mathbf{S})$ and $p(\mathbf{F})$ are drawn from the Gaussian distributions as:
\begin{equation}
	\begin{aligned}
		p(\mathbf{S}) = \prod_{i=1}^{D_{\mathbf{S}}} \mathcal{N}(S_{i} | 0, 1);~
		p(\mathbf{F}) = \prod_{i=1}^{D_{\mathbf{F}}} \mathcal{N}(F_{i} | 0, 1).
	\end{aligned}	
\end{equation} 
where $D_{\mathbf{S}}$ and $D_{\mathbf{F}}$ are the dimensions of $\mathbf{S}$ and $\mathbf{F}$, respectively. In the inference model, the variational approximations of the posteriors are described as:
\begin{equation}
	\begin{aligned}
			&q(\mathbf{S}|\mathbf{X}) = \prod_{i=1}^{D_{\mathbf{S}}} \mathcal{N}(\mu = \hat{\mu}_{\mathbf{S}_{i}}, \sigma^2 = \hat{\sigma}^2_{\mathbf{S}_{i}});~
			q(\mathbf{C}|\mathbf{X}) = \prod_{i=1}^{D_{\mathbf{C}}} \mathcal{N}(\mu = \hat{\mu}_{\mathbf{C}_i}, \sigma^2 = \hat{\sigma}^2_{\mathbf{C}_i}); \\
			&q(\mathbf{F}|\mathbf{X}) = \prod_{i=1}^{D_{\mathbf{F}}} \mathcal{N}(\mu = \hat{\mu}_{\mathbf{F}_i}, \sigma^2 = \hat{\sigma}^2_{\mathbf{F}_i}),
	\end{aligned}	
\end{equation}
where $D_{\mathbf{C}}$ is the dimension of $\mathbf{C}$, and $\hat{\mu}_{\mathbf{S}}, \hat{\mu}_{\mathbf{C}}, \hat{\mu}_{\mathbf{F}}$ and $\hat{\sigma}^2_{\mathbf{S}}, \hat{\sigma}^2_{\mathbf{C}}, \hat{\sigma}^2_{\mathbf{F}}$ are the parameters of means and variances in the Gaussian distributions parameterised by neural networks.
	
In the generative model, we utilise the Monte Carlo (MC) sampling strategy to sample the distribution $\mathbf{C}$ based on the Conditional Variational AutoEncoder network (CVAE)~\cite{sohn2015learning} such that the latent representation of $\mathbf{C}$ is generated from the distribution $\mathbf{X}$:
\begin{equation}
	\begin{aligned}
		p(\mathbf{C}) \backsim p(\mathbf{C} |\mathbf{X}).
	\end{aligned}	
\end{equation}

Furthermore, the generative models for $W$ and $\mathbf{X}$ with the three latent representations are formalised as:
\begin{equation}
	\begin{aligned}
		p(W|\mathbf{S},\mathbf{C}) = Bern(\sigma(\psi_1(\mathbf{S}, \mathbf{C})));~
		p(\mathbf{X}|\mathbf{S}, \mathbf{F}) = \prod_{i=1}^{D_{\mathbf{X}}} p(X_i|\mathbf{S}, \mathbf{C}),
	\end{aligned}	
\end{equation}where $\psi_1(\cdot)$ is a function parameterised by neural networks, $\sigma(\cdot)$ is the logistic function and $Bern(\cdot)$ is the function of Bernoulli distribution.
	
In our generative model, the latent representation for the outcome $Y$ is based on the data type of $Y$. For the outcome $Y$ with continuous values, we use a Gaussian distribution with its mean and variance parameterised by a pair of independent neural networks, \ie $p(Y | w = 0, \mathbf{C}, \mathbf{F})$ and $p(Y | w = 1, \mathbf{C}, \mathbf{F})$. Thus,  the continuous $Y$ is modelled by:
\begin{equation}
	\begin{aligned}
		&p(Y | W, \mathbf{C}, \mathbf{F}) = \mathcal{N}(\mu = \hat{\mu}_{Y}, \sigma^2 = \hat{\sigma}^2_{Y}),\\
		&\hat{\mu}_{Y} = W \cdot \psi_2(\mathbf{C}, \mathbf{F}) + (1-W) \cdot \psi_3(\mathbf{C},\mathbf{F});\\
		&\hat{\sigma}^2_{Y} = W \cdot \psi_4(\mathbf{C}, \mathbf{F}) + (1-W) \cdot \psi_5(\mathbf{C}, \mathbf{F}),
	\end{aligned}	
\end{equation}
where $\psi_2(\cdot), \psi_3(\cdot), \psi_4(\cdot)$ and $\psi_5(\cdot)$ are neural networks parameterised by their own parameters. 

For the outcome $Y$ with binary values, a Bernoulli distribution function based on neural networks is employed to model it and described as:
\begin{equation}
	\begin{aligned}
		&p(Y|W,\mathbf{C}, \mathbf{F}) = Bern(\sigma(\psi_6(W, \mathbf{C}, \mathbf{F}))),
	\end{aligned}	
\end{equation}
where $\psi_6(\cdot)$ is the same with the function $\psi_1$. These parameters of neural networks can be approximated by maximising the Evidence lower bound (ELBO) $\mathcal{L}_{ELBO}$:
\begin{equation}
	\begin{aligned}
		\mathcal{L}_{ELBO} (\mathbf{X}, W, Y)=~ &\mathbb{E}_{q}[\log p(\mathbf{X}|\mathbf{S}, \mathbf{C}, \mathbf{F})] - D_{KL}[q(\mathbf{S}|\mathbf{X})||p(\mathbf{S})] \\&
		- D_{KL}[q(\mathbf{C}|\mathbf{X})||p(\mathbf{C}|\mathbf{X})] - D_{KL}[q(\mathbf{F}|\mathbf{X})||p(\mathbf{F})],
	\end{aligned}	
\end{equation}
where the decoder $p(\mathbf{C}|\mathbf{X})$ is  to ensure that the latent representation $\mathbf{C}$ captures as much information of $\mathbf{X}$ as possible. 

To ensure that the treatment $W$ can be recovered from the latent representations $\mathbf{S}$ and $\mathbf{C}$, and the outcome $Y$ can be recovered from  the latent representations $\mathbf{C}$ and $\mathbf{F}$,  two auxiliary predictors are added and the objective function of DVAE.CIV can be formalised as:
\begin{equation}
	\begin{aligned}
		\mathcal{L}_{DVAE.CIV} =~ &-\mathcal{L}_{ELBO} (\mathbf{X}, W, Y) + \alpha \mathbb{E}_{q}[\log q(W|\mathbf{S}, \mathbf{C})] \\ &
		+ \beta \mathbb{E}_{q}[\log q(Y|W, \mathbf{C}, \mathbf{F})],
	\end{aligned}	
\end{equation}where $\alpha$ and $\beta$ are the weights for the auxiliary predictors. 

After training the DVAE.CIV model, we get the CIV representation $\mathbf{S}$ and the conditioning set  representations $\{\mathbf{C, F}\}$ based on Theorem~\ref{theo:001}. For estimating conditional causal effects, we employ an IV-based prediction, DeepIV~\cite{hartford2017deep}, to implement this part, \ie we feed $\mathbf{S}$ and $\{\mathbf{C, F}\}$ into the DeepIV method for conditional causal effect estimation.

\section{Experiments}
\label{sec:Exp}
In this section, we evaluate the performance of the proposed DVAE.IV model by applying it to a set of synthetic datasets and three real-world datasets for  CACE$(W, Y)$ and average causal effect (ACE) estimation. The three real-world datasets include SchoolingReturns~\cite{card1993using}, Cattaneo~\cite{cattaneo2010efficient} and RHC~\cite{connors1996effectiveness} that are usually utilised in evaluating the methods of causal effect estimation from observational data. Details of the implementation of DVAE.CIV and the appendix are provided in the GitHub\textsuperscript{1}~\footnote{\textsuperscript{1}~\url{https://github.com/IRON13/DVAE.CIV}}. 

\subsection{Experimental Setup}
We compare the DVAE.CIV against the famous estimators in conditional causal effect estimation that are widely utilised in causal inference from observational data. Note that the ACE can be obtained by averaging the CACE$(W, Y)$ of all individuals. These compared causal effect estimators are introduced in the following. 

\paragraph{Compared causal effect estimators.} We compare our proposed DVAE.CIV with two Variational AutoEncoder based  (VAE-based) causal effect estimators, three tree-based causal effect estimators, two machine learning based (ML-based) causal effect estimators, and three IV-based causal effect estimators.  The two VAE-based causal effect estimators are Causal Effect Variational AutoEncoder (CEVAE)~\cite{louizos2017causal} and Treatment Effect estimation by Disentangled Variational AutoEncoder (TEDVAE)~\cite{zhang2021treatment}). The three tree-based causal effect estimators are the standard Bayesian Additive Regression Trees (BART)~\cite{hill2011bayesian}, causal random forest (CF)~\cite{wager2018estimation} and causal random forest for IV regression (CFIVR)~\cite{athey2019generalized}. Note that CFIVR also belongs to IV-based estimators. The two ML-based causal effect estimators are double machine learning (DML)~\cite{chernozhukov2018double} and doubly robust learning (DRL)~\cite{foster2019orthogonal}. The three IV-based causal effect estimators are DeepIV~\cite{hartford2017deep}, orthogonal instrumental variable (OrthIV)~\cite{syrgkanis2019machine} and double machine learning based IV (DMLIV)~\cite{chernozhukov2018double}.

\paragraph{Remarks.} The five estimators TEDVAE, BART, CF, DML and DRL rely on the assumption of unconfoundedness~\cite{imbens2015causal} (\ie no latent confounders in data), so the five estimators cannot deal with the case with the data with latent confounders. CEVAE can deal with latent confounders, but it requires that all measured variables are proxy variables of the latent confounders, while our DVAE.CIV model does not have the restriction. The IV-based estimators CFIVR, DeepIV, OrthIV and DMLIV require a known IV that is nominated based on domain knowledge, but the nominated IV usually is not a valid IV and thus may result in a wrong conclusion as argued in Introduction. 

\paragraph{Implementation details.} We use \textit{Python} and the libraries including \textit{pytorch}~\cite{paszke2019pytorch}, \textit{pyro}~\cite{bingham2019pyro} and \textit{econml} to implement DVAE.CIV. In our experiments, the dimension of latent representations is set as  $\left |S \right|=1$, $\left |C \right|=5$ and $\left |F \right|=5$, respectively. The implementation of CEVAE is based on the \textit{Python} library \textit{pyro}~\cite{bingham2019pyro} and the code of TEDVAE is from the authors' GitHub\textsuperscript{2}\footnote{\textsuperscript{2}~\url{https://github.com/WeijiaZhang24/TEDVAE}}. For BART, we use the implementation in the \textit{R} package \textit{bartCause}~\cite{hill2011bayesian}. For CF and CFIVR, we use the implementations in the \textit{R} functions \emph{causal}$\_$\emph{forest} and \emph{instrumental}$\_$\emph{forest} in the \textit{R} package \emph{grf}~\cite{athey2019generalized}, respectively. The implementations of DML, DRL, DeepIV, OrthIV and DMLIV are from the \textit{Python} package \textit{encoml}.

\paragraph{Evaluation metrics.} For performance evaluation, two commonly used metrics are employed in our experiments.  For the synthetic datasets, we use absolute error of average causal effect~\cite{hill2011bayesian}, \ie $\varepsilon_{ACE} = |ACE-\hat{ACE}|$ where $ACE$ is the true causal effect and $\hat{ACE}$ is the estimated causal effect, and Precision of the Estimation of Heterogeneous Effect (PEHE, it is used to evaluate the CACE estimations.)~\cite{hill2011bayesian,louizos2017causal} $\sqrt{\varepsilon_{PEHE}} = \sqrt{\mathbf{E}(((y_1 - y_0)-(\hat{y}_1 - \hat{y}_0))^{2})}$ where $y_1, y_0$ are the true outcomes and $\hat{y}_1, \hat{y}_0$ are the predicted outcomes,  to assess the performance of all methods in terms of the causal effect estimation. Lower values of both metrics indicate better performance. For multiple replications, we present the mean with standard deviation. For the three real-world datasets,  we use the reference causal effect in the literature as the baseline to evaluate the performance of all estimators since there is no ground truth causal effect available.  
        
\subsection{Simulation Study}
\label{subsec:Simu}
It is challenging to evaluate a causal effect estimation method with real-world data since there is no ground truth in the real-world data. In this section, we design simulation studies to evaluate the performance of our proposed DVAE.CIV method in the case that there exists a latent confounder $U$ affecting both $W$ and $Y$, and there exists a CIV and its conditioning set in the synthetic datasets. 

We use a causal DAG $\mathcal{G}$ provided in the appendix to generate synthetic datasets with a range of sample sizes: 2k, 6k, 10k, and 20k. In the causal DAG $\mathcal{G}$, $\mathbf{X}=\{S, X_1, X_2, X_3, X_4, X_5\}$ is the set of measured covariates and $\mathbf{U}=\{U, U_1, U_2, U_3, U_4\}$ is the set of latent confounders in which $U$ affects both $W$ and $Y$. Note that $S$ is a CIV conditioning on the set $\{X_1, X_2\}$ for all synthetic datasets. Moreover, the data generation process allows the synthetic datasets to have the true individual causal effect. We provide the details of the synthetic data generating process in the appendix. In our experiments, the IV-based estimators OrthIV, DMLIV,  DeepIV and CFIVR utilise the true CIV $S$ and the conditioning set  $\{X_1, X_2\}$ as input for causal effect estimation. 

\begin{table}[ht]
	\centering
	\setlength{\tabcolsep}{1mm}{}
	\caption{The out-of-sample absolute error $\varepsilon_{ACE}$ (mean$\pm$std) over 30 synthetic datasets. The best results are highlighted in boldface and the runner-up results are \underline{underlined}. DVAE.CIV is the runner-up on all synthetic datasets, and it relies on the least domain knowledge among all estimators compared since it learns and disentangles the representations of CIV and its conditioning set from data directly.} 
	\label{tab:syn_ACEerror} 
	\begin{tabular}{|cc|c|c|c|c|}
	\hline
	\multicolumn{2}{|c|}{Samples}                              & 2k                                  & 6k                                        & 10k                      & 20k                      \\ \hline
	\multicolumn{2}{|c|}{Estimators}                              & $\varepsilon_{ACE}$      & $\varepsilon_{ACE}$      & $\varepsilon_{ACE}$      & $\varepsilon_{ACE}$      \\ \hline
	\multicolumn{1}{|c|}{\multirow{2}{*}{ML-based}}   & DML    & 5.507$\pm$0.387                & 5.624$\pm$0.182              & 5.619$\pm$0.122          & 5.633$\pm$0.096          \\ \cline{2-6} 
	\multicolumn{1}{|c|}{}                            & DRL    & 5.746$\pm$0.404                 & 5.833$\pm$0.186                  & 5.825$\pm$0.156          & 5.860$\pm$0.106          \\ \hline
	\multicolumn{1}{|c|}{\multirow{3}{*}{tree-based}} & BART   & 3.586$\pm$0.179          & 3.596$\pm$0.090                  & 3.613$\pm$0.065 & 3.622$\pm$0.060\\ \cline{2-6} 
	\multicolumn{1}{|c|}{}                            & CF     & 3.226$\pm$0.342& 3.246$\pm$0.141& 3.274$\pm$0.127& 3.312$\pm$0.074  \\ \hline       
	\multicolumn{1}{|c|}{\multirow{2}{*}{VAE-based}}  & CEVAE  & 5.595 $\pm$0.455  & 5.652$\pm$0.183       & 5.631$\pm$0.179          & 5.726$\pm$0.123          \\ \cline{2-6} 
	\multicolumn{1}{|c|}{}                            & TEDVAE & 5.615$\pm$0.455 & 5.655$\pm$0.181    & 5.634$\pm$0.172 & 5.696$\pm$0.112 \\ \hline
	\multicolumn{1}{|c|}{\multirow{4}{*}{IV-based}}   & OrthIV & 2.212$\pm$1.260        & 1.952$\pm$0.585                & 1.792$\pm$0.607          & 1.974$\pm$0.419          \\ \cline{2-6} 
	\multicolumn{1}{|c|}{}                            & DMLIV  & 2.170$\pm$1.189             & 1.888$\pm$0.572            & 1.790$\pm$0.626          & 1.971$\pm$0.432      \\\cline{2-6} 
	\multicolumn{1}{|c|}{}           & DeepIV & \textbf{0.352$\pm$0.180}  & 0.632$\pm$0.245                  & 0.726$\pm$0.315          & 0.757$\pm$0.354          \\ \cline{2-6} 
	 \multicolumn{1}{|c|}{}                          & CFIVR  & 1.228$\pm$0.949& \textbf{0.504$\pm$0.369} & \textbf{0.543$\pm$0.474} & \textbf{0.415$\pm$0.307} \\ \hline
	\multicolumn{2}{|c|}{DVAE.CIV}   & \underline{0.577$\pm$0.117}     & \underline{0.577$\pm$0.064}      & \underline{0.561$\pm$0.075}        & \underline{0.512$\pm$0.091}        \\ \hline
	\end{tabular}
\end{table}

\begin{table}
	\centering
	\setlength{\tabcolsep}{1mm}{}
	\caption{The out-of-sample  $\sqrt{\varepsilon_{PEHE}}$ (mean$\pm$std) over 30 synthetic datasets. The lowest $\sqrt{\varepsilon_{PEHE}}$ are highlighted in boldface and the runner-up results are \underline{underlined}.  DVAE.CIV is in the runner-up results on the first two groups of synthetic datasets and achieves the third smallest $\sqrt{\varepsilon_{PEHE}}$ on the last four groups of synthetic datasets. It's worth mentioning that DVAE.CIV obtains the lowest standard deviation on all synthetic datasets. } 
	\label{tab:syn_PEHE}
	\begin{tabular}{|cc|c|c|c|c|}
	\hline
	\multicolumn{2}{|c|}{Samples}                              & 2k                             & 6k                          & 10k                      & 20k                      \\ \hline
	\multicolumn{2}{|c|}{Estimators}                             & $\sqrt{\varepsilon_{PEHE}}$     & $\sqrt{\varepsilon_{PEHE}}$    & $\sqrt{\varepsilon_{PEHE}}$      & $\sqrt{\varepsilon_{PEHE}}$    \\ \hline
	\multicolumn{1}{|c|}{\multirow{2}{*}{ML-based}}   & DML    & 5.484$\pm$0.382          & 5.584$\pm$0.167& 5.580$\pm$0.128 & 5.609$\pm$0.105\\ \cline{2-6} 
	\multicolumn{1}{|c|}{}                            & DRL    & 5.701$\pm$0.408 & 5.773$\pm$0.179 & 5.767$\pm$0.156          & 5.815$\pm$0.112\\ \hline
	\multicolumn{1}{|c|}{\multirow{3}{*}{tree-based}} & BART   & 4.791$\pm$0.205 & 4.790$\pm$0.083   & 4.789$\pm$0.072 & 4.790$\pm$0.060\\ \cline{2-6} 
	\multicolumn{1}{|c|}{}                            & CF     & 3.483$\pm$0.319 & 3.500$\pm$0.134 & 3.523$\pm$0.120& 3.554$\pm$0.070\\ \hline       
	\multicolumn{1}{|c|}{\multirow{2}{*}{VAE-based}}  & CEVAE  & 6.093$\pm$0.396 & 6.138$\pm$0.175 & 6.107$\pm$0.160          & 6.192$\pm$0.112          \\ \cline{2-6} 
	\multicolumn{1}{|c|}{}                            & TEDVAE & 6.111$\pm$0.392  &  6.138$\pm$0.177  & 6.110$\pm$0.158 & 6.167$\pm$0.103 \\ \hline
	\multicolumn{1}{|c|}{\multirow{4}{*}{IV-based}}   & OrthIV & 3.070$\pm$0.718     & 2.798$\pm$0.299  & 2.734$\pm$0.256 & 2.795$\pm$0.218 \\ \cline{2-6} 
	\multicolumn{1}{|c|}{}                            & DMLIV  & 3.027$\pm$0.682 & 2.767$\pm$0.278  & 2.736$\pm. $0.268 &2.794$\pm$0.221      \\ \cline{2-6} 
	 \multicolumn{1}{|c|}{}                          & DeepIV & \textbf{2.396$\pm$0.054} & \underline{2.412$\pm$0.042}  &\textbf{2.418$\pm$0.060} & \underline{2.425$\pm$0.065} \\ \cline{2-6} 
	 \multicolumn{1}{|c|}{}                          & CFIVR  & 3.016$\pm$0.658& \textbf{2.421$\pm$0.235}  & \underline{2.423$\pm$0.351} & \textbf{2.203$\pm$0.145} \\ \hline
	\multicolumn{2}{|c|}{DVAE.CIV}   & \underline{2.448$\pm$0.044}    & {2.460$\pm$0.037}    & {2.452$\pm$0.024}        & {2.442$\pm$0.025}        \\ \hline
	\end{tabular}
\end{table}

To provide a reliable assessment, we repeatedly generate 30 synthetic datasets for each sample size setting and utilize the aforementioned metrics to evaluate the performance of the DVAE.CIV against the compared estimators with respect to the task of ACE estimation and CACE estimation from data with latent confounders. For each dataset, we randomly take 70\% of samples for training and 30\% for testing. The results of all estimators with respect to the ACE estimations and CACE estimations measured by the metrics $\varepsilon_{ACE}$ and $\sqrt{\varepsilon_{PEHE}}$ in the out-of-sample set are provided in Tables~\ref{tab:syn_ACEerror} and~\ref{tab:syn_PEHE}, respectively. The out-of-sample set is on testing samples, and the within-sample set is on training samples. The results of the within-sample set are provided in the appendix.

\paragraph{Results.} By analysing the experiment results in Table~\ref{tab:syn_ACEerror}, we have the following observations: (1) the ML-based and VAE-based estimators, DML, DRL, CEVAE and TEDVAE have the largest  $\varepsilon_{ACE}$ because the confounding bias caused by confounders and the latent confounder $U$ is not adjusted at all. (2) the tree-based estimators, BART and CF have the second largest $\varepsilon_{ACE}$ as they fail to deal with the confounding bias caused by the latent confounder $U$. (3) the IV-based estimators including DVAE.CIV significantly outperform the other estimators including DML, DRL, BART, CF, CEVAE and TEDVAE.  (4) DVAE.CIV is the second best performer on all synthetic datasets and its performance is comparable with CFIVR and DeepIV. (5) as the sample size increases, the standard deviation of most estimators including DVAE.CIV decreases significantly. It's worth mentioning that DVAE.CIV requires the least domain knowledge among all estimators since it only relies on the assumption that there exists a CIV and the conditioning set (maybe an empty set). This is very important in practice, as in many real-world applications, there is rarely sufficient prior knowledge for nominating a valid IV.

\begin{table}[ht]
	\centering
	\setlength{\tabcolsep}{1.5mm}{}
	\caption{Estimated ACEs by all methods on the three real-world datasets. We highlight the estimated causal effects within the empirical interval on SchoolingReturns and Cattaneo. We use `-' to indicate that an IV-based estimator does not work on Cattaneo and RHC since there is not a nominated IV. Note that all estimators on RHC obtain a consistent result.} 
	\label{tab:results}
	\begin{tabular}{|cc|c|c|c|}
	\hline
	\multicolumn{2}{|c|}{Samples}                              &  SchoolingReturns                                           & Cattaneo       &RHC               \\ \hline
	\multicolumn{1}{|c|}{\multirow{2}{*}{ML-based}}   & DML       &  -0.0227       & -150.21        &   \textbf{0.0244}    \\ \cline{2-5} 
	\multicolumn{1}{|c|}{}                            & DRL        &   -0.0154     & -164.32       & \textbf{0.0447}    \\ \hline
	\multicolumn{1}{|c|}{\multirow{3}{*}{tree-based}} & BART   & -0.0384        & -172.53       &  \textbf{0.0381}      \\ \cline{2-5} 
	\multicolumn{1}{|c|}{}                            & CF &   \textbf{0.1400}      &   \textbf{-232.33}    & \textbf{0.0278}        \\ \hline 
	\multicolumn{1}{|c|}{\multirow{2}{*}{VAE-based}}   & CEVAE  & 0.02617   &  \textbf{-221.23}   & \textbf{0.0322}       \\ \cline{2-5} 
	\multicolumn{1}{|c|}{}                            & TEDVAE & 0.0029     &   \textbf{-228.65}     &  \textbf {0.0293}     \\ \hline 
	\multicolumn{1}{|c|}{\multirow{4}{*}{IV-based}}   & OrthIV & 1.3180    &   -     & -         \\ \cline{2-5} 
	\multicolumn{1}{|c|}{}                            & DMLIV &  1.2806   &  -    &  -        \\ \cline{2-5} 
	\multicolumn{1}{|c|}{}           & DeepIV&  0.0328    & -  &  -         \\ \cline{2-5} 
	\multicolumn{1}{|c|}{}                          & CFIVR & 1.1510     &  -       &  -        \\ \hline 
	\multicolumn{2}{|c|}{DVAE.CIV}   & \textbf{0.1855}    &   \textbf{-224.79}     & \textbf{0.0414}        \\ \hline 
	\end{tabular}
\end{table}

From the results in Table \ref{tab:syn_PEHE}, we can conclude that (1) the ML-based, tree-based, and VAE-based estimators have the worst performance with respect to conditional causal effect estimations. (2) Among the IV-based estimators, DeepIV achieves the best performance on the first two groups of synthetic datasets and the second-best performance on the other datasets, and CFIVR obtains the best performance on the last four groups of synthetic datasets and the second-best performance on the first two groups of synthetic datasets. (3) DVAE.CIV obtains the second-best performance on all synthetic datasets. (4) The standard deviation of DVAE.CIV is the smallest on all datasets, and as the sample size increases, the standard deviation of DVAE.CIV reduces significantly. These conclusions demonstrate that DVAE.CIV can learn and disentangle the representations of the CIV and its conditioning set for CACE estimation from data with latent confounders.

In conclusion, DVAE.CIV achieves competitive performance compared to state-of-the-art causal effect estimators while requiring the least prior knowledge in ACE and CACE estimations from observational data with latent confounders.   

\subsection{Experiments on Three Real-World Datasets}
\label{sec:realworlddatasets}
We selected three real-world datasets with their empirical causal effect values available and commonly used in the literature to assess the performance of DVAE.CIV in ACE estimations. We did not conduct experiments on CACE estimation on the three datasets since there were no ground truth or empirical estimates of CACEs available for these datasets. The three real-world datasets are SchoolingReturns~\cite{card1993using}, Cattaneo~\cite{cattaneo2010efficient}, and RHC~\cite{connors1996effectiveness}. These datasets are widely utilized in the evaluation of either IV estimators or data-driven causal effect estimators~\cite{guo2020survey}. Note that SchoolingReturns has a nominated CIV, and the last two datasets do not have a nominated IV for causal effect estimation. Thus, we only compared the DVAE.CIV model with all the aforementioned estimators on SchoolingReturns and the ML-based, tree-based, and VAE-based estimators on both Cattaneo and RHC datasets.

\paragraph{SchoolingReturns.} The dataset is from the national longitudinal survey of youth (NLSY), a well-known dataset of US young employees, aged range from 24 to 34~\cite{card1993using}. The dataset has 3,010 samples and 19 variables~\cite{card1993using}. The  variable of the education of employees is the treatment variable, and the variable of the raw wages in 1976 (in cents per hour)  is the outcome variable. The dataset was collected to study the causal effect of education on earnings. Note that the variable of geographical proximity to a college, \ie \emph{nearcollege} is nominated to be an IV by Card~\cite{card1993using}. The empirical estimate ACE$(W, Y) = 0.1329$ with 95\% confidence interval (0.0484, 0.2175) is from~\cite{verbeek2008guide} and used as the reference value.

\paragraph{Cattaneo.} The dataset has the birth weights of 4,642 singleton births with 20 variables~(\cite{cattaneo2010efficient})  that were collected from Pennsylvania, USA for the study of the average of maternal smoking status during pregnancy ($W$) on a baby's birth weight ($Y$, in grams). The dataset contains several covariates: mother's age, mother's marital status, an indicator for the previous infant where the newborn died, mother's race, mother's education, father's education, number of prenatal care visits, months since last birth, an indicator of firstborn infant and indicator of alcohol consumption during pregnancy. The authors~\cite{cattaneo2010efficient} found a strong negative effect of maternal smoking on the weights of babies, i.e., about $200g$ to $250g$ lighter for a baby with a mother smoking during pregnancy.

\paragraph{Right Heart Catheterization (RHC).} RHC is a real-world dataset obtained from an observational study regarding a diagnostic procedure for the management of critically ill patients~\cite{connors1996effectiveness}. The RHC dataset can be downloaded from the $\mathbf{R}$ package \emph{Hmisc}\textsuperscript{3}\footnote{\textsuperscript{3}~\url{https://CRAN.R-project.org/package=Hmisc}}. The dataset contains 2,707 samples with 72 covariates~\cite{connors1996effectiveness,loh2020confounder}. RHC was for investigating the adult patients who participated in the Study to Understand Prognoses and Preferences for Outcomes and Risks of Treatments (SUPPORT). The treatment variable $W$ is whether a patient received an RHC within 24 hours of admission, and the outcome variable $Y$ is whether a patient died at any time up to 180 days after admission. Note that the empirical conclusion is that applying RHC leads to higher mortality within 180 days than not applying RHC~\cite{connors1996effectiveness}.

\paragraph{Results.} All results on the three real-world datasets are reported in Table~\ref{tab:results}. From Table~\ref{tab:results}, we make the following observations: (1) the estimated causal effects by DVAE.CIV and CF on SchoolingReturns and Cattaneo fall within the empirical intervals, while DML, DRL, and BART provide an opposite estimate to the empirical value on SchoolingReturns; (2) as there is no nominated IV on Cattaneo and RHC, the estimators OrthIV, DMLIV, DeepIV, and CFIVR do not work on both datasets; (3) all estimators, including DVAE.CIV, obtain a consistent estimation on the RHC data, and they reach the same conclusion as the empirical conclusion~\cite{connors1996effectiveness}. These observations further confirm that DVAE.CIV is capable of removing the bias between $W$ and $Y$ in real-world datasets.

In conclusion, our simulation studies show the high performance of DVAE.CIV in ACE and CACE estimations from data with latent confounders, and our experiments on three real-world datasets further confirm the capability of DVAE.CIV in ACE estimation from observational data.

\paragraph{Limitations.} The performance of DVAE.CIV relies on the assumptions made in this work and the assumptions on the VAE model. Note that the identification of the VAE model~\cite{khemakhem2020variational} is an important issue for our proposed DVAE.CIV model. When some of the assumptions or the VAE identification do not hold, DVAE.CIV may obtain an inconsistent conclusion. To obtain a consistent conclusion, it would be better to conduct a sensitivity analysis~\cite{pearl2009causality,imbens2015causal} together with DVAE.CIV to achieve a reliable conclusion in real-world applications.

\section{Conclusion}
\label{sec:cons}
It is a crucial challenge to deal with the bias caused by latent confounders in conditional causal effect estimations from observational data. IV-based methods allow us to remove such confounding bias in an effective way, but it relies on a nominated IV/CIV based on domain knowledge. In this paper, we propose an efficient approach, DVAE.CIV for conditional causal effect estimations from observational data with latent confounders. The DVAE.CIV utilizes the advantages of deep generative models for learning the representations of a CIV and its conditioning set from data with latent confounders. We theoretically show the soundness of the DVAE.CIV model. The effectiveness and potential of the DVAE.CIV are demonstrated by extensive experiments. In simulation studies, DVAE.CIV achieves competitive performance against state-of-the-art estimators that require extra prior knowledge in ACE and CACE estimation from data with latent confounders. The experimental results on three real-world datasets show the superiority of the DVAE.CIV model on ACE estimation over the existing estimators.

\subsubsection{Acknowledgements} This work has been supported by the Australian Research Council (grant number: DP200101210 and DP230101122).

\newpage

\appendix
\noindent{\Large  \textbf{Appendix}}\\
In this Appendix, we provide additional graphical notations and definitions, details of synthetic and real-world datasets, and more experimental results. 

\section{Preliminaries}

\paragraph{Paths.} In a graph $\mathcal{G}$, a path $\pi$ between $V_{1}$ and $V_{p}$ consists of a sequence of distinct nodes $\langle V_{1}, \dots, V_{p}\rangle$ with every pair of successive nodes being adjacent. A node $V$ lies on the path $\pi$ if $V$ belongs to the sequence $\langle V_{1}, \dots, V_{p}\rangle$. A path $\pi$ is a directed or causal path if all edges along it are directed such as $V_{1} \rightarrow\ldots \rightarrow V_{p}$. 

\paragraph{Ancestral relationships.} In a DAG $\mathcal{G}$, $V_i$ is a parent of $V_j$ (and $V_j$ is a child of $V_i$) if $V_i \rightarrow V_j$ appears in this graph.  In a directed path $\pi$, $V_i$ is an ancestor of $V_j$ and $V_j$ is a descendant of $V_i$ if all arrows along $\pi$ point to $V_j$.

\vspace{5pt}
In graphical causal modelling, the assumptions of Markov property, faithfulness and causal sufficiency are often involved to discuss the relationship between the causal graph and the distribution of the data.

\begin{definition}[Markov property~\cite{pearl2009causality}]
	\label{Markov condition}
	Given a DAG $\mathcal{G}=(\mathbf{V}, \mathbf{E})$ and the joint probability distribution of $\mathbf{V}$ $(prob(\mathbf{V}))$, $\mathcal{G}$ satisfies the Markov property if for $\forall V_i \in \mathbf{V}$, $V_i$ is probabilistically independent of all of its non-descendants, given the set of parents $V_i$.
\end{definition}

\begin{definition}[Faithfulness~\cite{spirtes2000causation}]
	\label{Faithfulness}
	A DAG $\mathcal{G}=(\mathbf{V}, \mathbf{E})$ is faithful to a joint distribution $prob(\mathbf{V})$ over the set of variables $\mathbf{V}$ if and only if every independence present in $prob(\mathbf{V})$ is entailed by $\mathcal{G}$ and satisfies the Markov property. A joint distribution $prob(\mathbf{V})$ over the set of variables $\mathbf{V}$ is faithful to the DAG $\mathcal{G}$ if and only if the DAG $\mathcal{G}$ is faithful to the joint distribution $prob(\mathbf{V})$.
\end{definition}

\begin{definition}[Causal sufficiency~\cite{spirtes2000causation}]
	A given dataset satisfies causal sufficiency if in the dataset for every pair of observed variables, all their common causes are observed.
\end{definition}

In a DAG, d-separation is a graphical criterion that enables the identification of conditional independence between variables entailed in the DAG when the Markov property, faithfulness and causal sufficiency are satisfied~\cite{pearl2009causality,spirtes2000causation}.

\begin{definition}[d-separation~\cite{pearl2009causality}]
	\label{d-separation}
	A path $\pi$ in a DAG $\mathcal{G}=(\mathbf{V}, \mathbf{E})$ is said to be d-separated (or blocked) by a set of nodes $\mathbf{Z}$ if and only if
	(i) $\pi$ contains a chain $V_i \rightarrow V_k \rightarrow V_j$ or a fork $V_i \leftarrow V_k \rightarrow V_j$ such that the middle node $V_k$ is in $\mathbf{Z}$, or
	(ii) $\pi$ contains a collider $V_k$ such that $V_k$ is not in $\mathbf{Z}$ and no descendant of $V_k$ is in $\mathbf{Z}$.
	A set $\mathbf{Z}$ is said to d-separate $V_i$ from $V_j$ ($ V_i \indep_d V_j\mid\mathbf{Z}$) if and only if $\mathbf{Z}$ blocks every path between $V_i$ to $V_j$. otherwise they are said to be d-connected by $\mathbf{Z}$, denoted as $V_i\nindep_d V_j\mid\mathbf{Z}$.
\end{definition}

\section{Experiments}

\begin{figure}[t]
	\centering
	\includegraphics[scale=0.485]{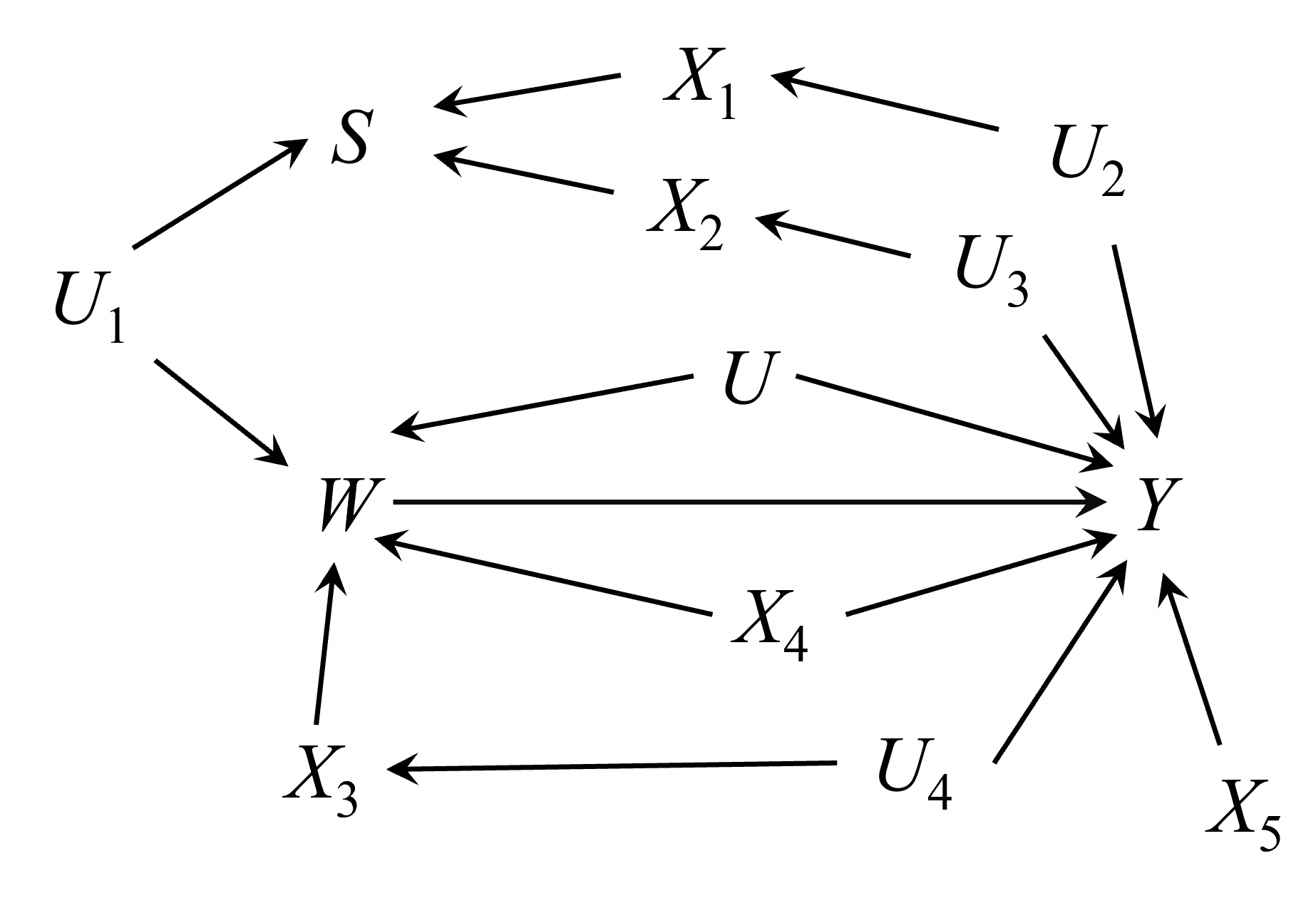}
	\caption{The true causal DAG with a latent confounder $U$ between $W$ and $Y$ is utilised to generate the synthetic datasets in Simulation Study.  $\mathbf{X}= \{S, X_1, X_2, X_3, X_4, X_5\}$ are pretreatment variables, and $\mathbf{U}= \{U, U_1, U_2, U_3, U_4\}$ are five latent confounders. Note that $S$ is a CIV conditioning on $\{X_1, X_2\}$.} 
	\label{fig:figure_001}
\end{figure}

\subsection{Simulation Study}
The simulated datasets are generated from the true DAG in Fig.~\ref{fig:figure_001}, and the specifications of the data generation are as follows: $U, U_1, U_2, U_3, U_4 \sim N(0, 1)$ and $\epsilon_{1, 2, 3, s} \sim N(0, 0.5)$, where $N(,)$ denotes the normal distribution. $X_1 \sim N(0, 1) +0.5* U_2 +\epsilon_{1}$,  $X_2 \sim N(0, 1) +0.5* U_3 +\epsilon_{2}$, $X_3 \sim N(0, 1) +0.5* U_4 +\epsilon_{3}$, $S\sim N(0, 1) + 2*U_1 + 1.5*X_1 + 1.5*X_2 + \epsilon_{s}$,  $X_4 \sim N(1, 1)$,  and $X_5 \sim N(3, 1)$.

The treatment assignment $W$ is generated from $n$ ($n$ denotes the sample size) Bernoulli trials by using the assignment probability $P(W=1\mid U, U_1, X_3) = [1+exp\{2- 1*U - 1*U_1 -1*X_3-1*X_4\}]$. The potential outcome is generated from $Y_{W} = 2 + 2*W + 2*U + 2*U_3 + 2*U_4 + 1*X_4 + 1*X_5 +\epsilon_{W}$ where $\epsilon_{W}\sim N(0, 1)$. Note that true $ACE$ is fixed to 2 on all synthetic datasets.

\begin{table*}[t]
	\centering
	\begin{tabular}{|cl|c|c|c|c|}
		\hline
		\multicolumn{2}{|c|}{Samples}                              & 2k                                  & 6k                                        & 10k                      & 20k                      \\ \hline
		\multicolumn{2}{|c|}{Estimators}                              & $\varepsilon_{ACE}$       & $\varepsilon_{ACE}$      & $\varepsilon_{ACE}$      & $\varepsilon_{ACE}$      \\ \hline
		\multicolumn{1}{|c|}{\multirow{2}{*}{ML-based}}   & DML    & 5.507$\pm$0.387              & 5.624$\pm$0.182       & 5.619$\pm$0.122          & 5.633$\pm$0.096          \\ \cline{2-6} 
		\multicolumn{1}{|c|}{}                            & DRL    & 5.746$\pm$0.404          & 5.833$\pm$0.186   & 5.825$\pm$0.156   & 5.860$\pm$0.106          \\ \hline
		\multicolumn{1}{|c|}{\multirow{3}{*}{tree-based}} & BART   & 3.890$\pm$0.368    & 3.999$\pm$0.156   & 4.014$\pm$0.152 & 4.046$\pm$0.106\\ \cline{2-6} 
		\multicolumn{1}{|c|}{}                            & CF     & 3.218$\pm$0.325& 3.255$\pm$0.140& 3.277$\pm$0.131& 3.306$\pm$0.077  \\ \hline       
		\multicolumn{1}{|c|}{\multirow{2}{*}{VAE-based}}  & CEVAE  & 5.558 $\pm$0.407  & 5.698$\pm$0.194  & 5.640$\pm$0.172         & 5.706$\pm$0.112          \\ \cline{2-6} 
		\multicolumn{1}{|c|}{}                            & TEDVAE & 5.671$\pm$0.399 & 5.583$\pm$0.194   & 5.644$\pm$0.167 & 5.674$\pm$0.100\\ \hline
		\multicolumn{1}{|l|}{\multirow{4}{*}{IV-based}}   & OrthIV & 2.212$\pm$1.260     & 1.952$\pm$0.585        & 1.792$\pm$0.607          & 1.974$\pm$0.419          \\ \cline{2-6} 
		\multicolumn{1}{|l|}{}                            & DMLIV  & 2.170$\pm$1.189     & 1.888$\pm$0.572        & 1.790$\pm$0.626          & 1.971$\pm$0.432      \\\cline{2-6} 
		\multicolumn{1}{|l|}{}           & DeepIV & \textbf{0.352$\pm$0.180} & 0.632$\pm$0.245    & 0.727$\pm$0.315          & 0.757$\pm$0.354          \\ \cline{2-6} 
		\multicolumn{1}{|l|}{}                          & CFIVR  & 1.217$\pm$0.924&  \textbf{0.514$\pm$0.369}  & \textbf{0.552$\pm$0.461} & \textbf{0.416$\pm$0.296} \\ \hline
		\multicolumn{2}{|c|}{DRVAE.CIV}   & \underline{0.612$\pm$0.090}   & \underline{0.588$\pm$0.055}  & \underline{0.536$\pm$0.085}  & \underline{0.512$\pm$0.091}        \\ \hline
	\end{tabular}
	\caption{The within-sample absolute error $\varepsilon_{ACE}$ (mean$\pm$std) over 30 synthetic datasets. The best results are highlighted in boldface and the runner-up results are \underline{underlined}.} 
	\label{tab:syn_ACEerror} 
\end{table*}

\begin{table*}[ht]
	\centering
	\begin{tabular}{|cl|c|c|c|c|}
		\hline
		\multicolumn{2}{|c|}{Samples}                              & 2k                                          & 6k                                     & 10k                      & 20k                      \\ \hline
		\multicolumn{2}{|c|}{Estimators}                              & $\sqrt{\varepsilon_{PEHE}}$       & $\sqrt{\varepsilon_{PEHE}}$    & $\sqrt{\varepsilon_{PEHE}}$      & $\sqrt{\varepsilon_{PEHE}}$    \\ \hline
		\multicolumn{1}{|c|}{\multirow{2}{*}{ML-based}}   & DML    & 5.455$\pm$0.353  & 5.596$\pm$0.174 & 5.587$\pm$0.115& 5.588$\pm$0.090          \\ \cline{2-6} 
		\multicolumn{1}{|c|}{}                            & DRL    & 5.671$\pm$0.370 & 5.786$\pm$0.182  & 5.774$\pm$0.144& 5.794$\pm$0.101\\ \hline
		\multicolumn{1}{|c|}{\multirow{3}{*}{tree-based}} & BART   & 4.185$\pm$344 & 4.227$\pm$0.149 & 4.234$\pm$0.146 & 4.253$\pm$0.106\\ \cline{2-6} 
		\multicolumn{1}{|c|}{}                            & CF     & 3.475$\pm$0.301& 3.504$\pm$0.129& 3.522$\pm$0.121& 3.547$\pm$0.072\\ \hline       
		\multicolumn{1}{|c|}{\multirow{2}{*}{VAE-based}}  & CEVAE  & 6.061$\pm$0.352   & 6.149$\pm$0.178                & 6.115$\pm$0.149         & 6.173$\pm$0.101          \\ \cline{2-6} 
		\multicolumn{1}{|c|}{}                            & TEDVAE & 6.076$\pm$0.337 & 6.149$\pm$0.175      & 6.119$\pm$0.148         & 6.147$\pm$0.091\\ \hline
		\multicolumn{1}{|l|}{\multirow{4}{*}{IV-based}}   & OrthIV & 3.050$\pm$0.700    & 2.804$\pm$0.303   & 2.736$\pm$0.255  & 2.784$\pm$0.213 \\ \cline{2-6} 
		\multicolumn{1}{|l|}{}                            & DMLIV  & 3.009$\pm$0.664  & 2.772$\pm$0.280 & 2.738$\pm$0.268 & 2.784$\pm$0.216      \\ \cline{2-6} 
		\multicolumn{1}{|l|}{}           & DeepIV & \textbf{2.403$\pm$0.036} &\textbf{2.408$\pm$0.038}  & \underline{2.418$\pm$0.062} & \underline{2.425$\pm$0.065} \\ \cline{2-6} 
		\multicolumn{1}{|l|}{}                          & CFIVR  & 3.048$\pm$0.649 &2.457$\pm$0.252 & \textbf{2.432$\pm$0.355} & \textbf{2.328$\pm$0.144} \\ \hline
		\multicolumn{2}{|c|}{DRVAE.CIV}   & \underline{2.460$\pm$0.041}         & \underline{2.454$\pm$0.029}   & {2.449$\pm$0.027}        & {2.448$\pm$0.015}        \\ \hline
	\end{tabular}
	\caption{The within-sample  $\sqrt{\varepsilon_{PEHE}}$ (mean$\pm$std) over 30 synthetic datasets. The lowest $\sqrt{\varepsilon_{PEHE}}$ are highlighted in boldface and the runner-up results are \underline{underlined}.  } 
	\label{tab:syn_PEHE}
\end{table*}

The experimental results of $\varepsilon_{ACE}$ and $\sqrt{\varepsilon_{PEHE}}$ on within-samples are reported in Tables~\ref{tab:syn_ACEerror} and \ref{tab:syn_PEHE}, respectively.

\paragraph{Results.}  Tables~\ref{tab:syn_ACEerror} and \ref{tab:syn_PEHE} support the same conclusion drawn in the main text.

\subsection{Experiments on Three Real-World Datasets}
\paragraph{SchoolingReturns.} The dataset is from the national longitudinal survey of youth (NLSY), a well-known dataset of US young employees, aged range from 24 to 34~\cite{card1993using}. The treatment is the education of employees, and the outcome is raw wages in 1976 (in cents per hour). The data contains 3,010 individuals and 19 covariates. The covariates include experience (Years of labour market experience), ethnicity, resident information of an individual, age, nearcollege (whether an individual grew up near a 4-year college?), marital status, Father's educational attainment, Mother's educational attainment, and so on. A goal of the studies on this dataset is to investigate the causal effect of education on earnings. Card~\cite{card1993using} used geographical proximity to a college, \ie the covariate \emph{nearcollege} as an instrument variable. We take $ACE = 0.1329$ with 95\% conditional interval (0.0484, 0.2175) from~\cite{verbeek2008guide} as the reference causal effect.

\paragraph{Cattaneo.} The Cattaneo~(\cite{cattaneo2010efficient}) is usually used to study the ACE of maternal smoking status during pregnancy ($W$) on a baby's birth weight (in grams)\footnote{\url{http://www.stata-press.com/data/r13/cattaneo2.dta}}.
Cattaneo2 consists of the birth weights of 4,642 singleton births in Pennsylvania, USA~(\cite{almond2005costs,cattaneo2010efficient}). Cattaneo contains 864 smoking mothers ($W$=1) and 3,778 nonsmoking mothers ($W$=0). The dataset contains several covariates: mother's age, mother's marital status, an indicator for the previous infant where the newborn died, mother's race, mother's education, father's education, number of prenatal care visits, months since last birth, an indicator of firstborn infant and indicator of alcohol consumption during pregnancy. The authors~(\cite{almond2005costs}) found a strong negative effect of maternal smoking on the weights of babies, that is, about $200g$ to $250g$ lower for a baby with a mother smoking during pregnancy than for a baby without by statistical analysis on all covariates.

\paragraph{Right heart catheterization (RHC).} Right heart catheterization (RHC) is a real-world dataset obtained from an observational study regarding a diagnostic procedure for the management of critically ill patients~(\cite{connors1996effectiveness}). The RHC dataset can be downloaded from the $\mathbf{R}$ package \emph{Hmisc}\footnote{\url{https://CRAN.R-project.org/package=Hmisc}}. RHC contains information on hospitalised adult patients from five medical centres in the USA. These hospitalised adult patients participated in the Study to Understand Prognoses and Preferences for Outcomes and Risks of Treatments (SUPPORT). Treatment $W$ indicates whether a patient received an RHC within 24 hours of admission. The outcome $Y$ is whether a patient died at any time up to 180 days after admission. The original RHC dataset has 5,735 samples with 73 covariates. We preprocess the original data, as suggested by Loh et al.~(\cite{loh2020confounder}), and the final dataset contains 2,707 samples with 72 covariates.  Note that the empirical conclusion is that applying RHC leads to higher mortality within 180 days than not applying RHC~\cite{connors1996effectiveness}.


%
%
%
\bibliographystyle{splncs04}
\bibliography{sub_73_main}
\end{document}